\pdfoutput=1
\def\year{2022}\relax
\documentclass[letterpaper]{article} 
\usepackage{aaai22}  
\usepackage{times}  
\usepackage{helvet}  
\usepackage{courier}  
\usepackage[hyphens]{url}  
\usepackage{graphicx} 
\usepackage{natbib}  
\usepackage{caption} 
\DeclareCaptionStyle{ruled}{labelfont=normalfont,labelsep=colon,strut=off} 
\usepackage{amsmath}
\usepackage{amssymb}
\usepackage{amsthm}
\usepackage{dsfont}
\usepackage{graphicx}
\usepackage{subcaption}
\usepackage{centernot}
\usepackage{comment}
\usepackage{multirow}
\usepackage{booktabs}
\usepackage{tablefootnote}
\newtheorem{prop}{Proposition}

\usepackage{algorithm}
\usepackage{algorithmic}

\usepackage{newfloat}
\usepackage{listings}
\floatstyle{ruled}
\newfloat{listing}{tb}{lst}{}
\floatname{listing}{Listing}
\title{LUNAR: Unifying Local Outlier Detection Methods via Graph Neural Networks}
\author {
    Adam Goodge,\textsuperscript{\rm 1,3}
    Bryan Hooi, \textsuperscript{\rm 1,2}
    Ng See Kiong \textsuperscript{\rm 1,2}
    Ng Wee Siong \textsuperscript{\rm 3}
}
\affiliations {
    \textsuperscript{\rm 1} School of Computing, National University of Singapore\\
    \textsuperscript{\rm 2} Institute of Data Science, National University of Singapore \\
    \textsuperscript{\rm 3} Institute for Infocomm Research, A*STAR, Singapore\\
    adam.goodge@u.nus.edu,
    \{dcsbhk, seekiong\}@nus.edu.sg, wsng@i2r.a-star.edu.sg 
}

\begin{document}

\maketitle

\section*{Abstract}

Many well-established anomaly detection methods use the distance of a sample to those in its local neighbourhood: so-called `local outlier methods', such as LOF and DBSCAN. They are popular for their simple principles and strong performance on unstructured, feature-based data that is commonplace in many practical applications. However, they cannot learn to adapt for a particular set of data due to their lack of trainable parameters. In this paper, we begin by unifying local outlier methods by showing that they are particular cases of the more general message passing framework used in graph neural networks. This allows us to introduce learnability into local outlier methods, in the form of a neural network, for greater flexibility and expressivity: specifically, we propose LUNAR, a novel, graph neural network-based anomaly detection method. LUNAR learns to use information from the nearest neighbours of each node in a trainable way to find anomalies. We show that our method performs significantly better than existing local outlier methods, as well as state-of-the-art deep baselines. We also show that the performance of our method is much more robust to different settings of the local neighbourhood size.

\section{Introduction}\label{section:Introduction}

Unsupervised anomaly detection is the task of detecting anomalies within a set of data without relying on ground truth labels of known anomalies. It is an extremely important task in a wide range of practical applications and has therefore received a great amount of research interest. As anomalies tend to be much rarer than normal data, labelled anomalies are difficult to obtain in the quantity needed to adequately train supervised techniques.

Many well-established unsupervised methods detect anomalies by measuring the distance of a point to its nearest neighbouring points: so-called \textbf{local outlier methods}, such as \textsc{LOF} and \textsc{DBSCAN}. These methods are very popular in practice due to their straightforward principles and assumptions, as well as their interpretable outputs. In our experiments, we also find that their performance also holds up favourably against more recent, deep learning-based methods. The latter have to fully embed knowledge about normal and abnormal regions of the data space in their network parameters. They are mostly designed for highly structured, high-dimensional data such as images, but their performance often struggles for the less structured, feature-based data that is commonly used in many applications. As such, local outlier methods remain the default choice is many areas.

A range of local outlier methods have been developed, each with its own unique formulation and properties. However, many of them also share common characteristics with each other. In this paper, our first contribution is to unify local outlier methods under a simple, general framework based on the message passing scheme used in \textbf{graph neural networks} (GNN). We demonstrate that many popular methods, such as \textsc{KNN}, \textsc{LOF} and \textsc{DBSCAN}, can be seen as particular cases of this more general message passing framework.

Despite their popularity, local outlier methods lack the capacity to learn to optimise for or adapt to a particular set of data, e.g. through trainable parameters. Furthermore, in an unsupervised setting, there is no straightforward way to find optimal hyper-parameter settings, such as the number of nearest neighbours, which is extremely important and greatly affects performance. In this paper, we also propose a novel method named \textsc{LUNAR} (\textbf{L}earnable \textbf{U}nified \textbf{N}eighbourhood-based \textbf{A}nomaly \textbf{R}anking), which is based on the same message passing framework for local outlier methods but addresses their shortcomings by enabling \textbf{learnability} via graph neural networks.

In summary, we make the following contributions:

\begin{itemize}
    \item We show that many popular local outlier methods, such as \textsc{KNN}, \textsc{LOF} and \textsc{DBSCAN}, can be unified under a single framework based on graph neural networks.
    \item We use this framework to develop a novel, GNN-based anomaly detection method (\textsc{LUNAR}) which is more flexible and adaptive to a given set of data than local outlier methods due to its trainable parameters.
    \item We show that our method gives better performance\footnote{Code available at https://github.com/agoodge/LUNAR} than popular classical methods, including local outlier methods, as well as state-of-the-art deep learning-based methods in anomaly detection. We also show that its performance is much more robust to different settings of the local neighbourhood size than local outlier methods.
\end{itemize}

\section{Related Work}

\paragraph{Neighbourhood-based Anomaly Detection}

Besides notable examples like \textsc{OC-SVM} \cite{scholkopf2001estimating} and \textsc{IForest} \cite{liu2008isolation}, most classical anomaly detection methods directly measure the distance of a point to its nearest neighbours to detect anomalies, which we call `local outlier methods'. They rely on the assumption that anomalies are in sparse regions of the data space, far away from highly dense clusters of normal points. Points that are close to their neighbours are more likely to be normal themselves, whilst points far from their neighbours are more likely to be anomalies.

\textsc{KNN} \cite{angiulli2002fast} uses the distance to the $k^{th}$ nearest neighbour as the anomaly score. Alternatively, \textsc{DBSCAN} \cite{Ester1996}, which simultaneously learns to to cluster normal data while also detecting outliers, uses the number of points within a pre-defined distance.

Local Outlier Factor (\textsc{LOF}) \cite{Breunig2000} measures distances to define a density measure and compares this density to neighbouring points. Various extensions and variants of have been developed, including but not limited to: Local Outlier Probabilities (\textsc{LoOP}) \cite{kriegel2009loop}, Connectivity-based (\textsc{COF}) \cite{tang2002enhancing},  Local Correlation Integral (\textsc{LOCI}) \cite{papadimitriou2003loci}, Influenced Outlierness (\textsc{INFLO}) \cite{Jin2006}, and Subspace Outlier Detection (SOD) \cite{kriegel2009outlier}.

These methods suffer from a lack of learnability: they do not use the information in the training set to optimise model parameters for better anomaly scoring. Instead, they are based on pre-defined heuristics and hyper-parameters. These settings strongly influence performance, yet the optimal settings are very difficult to validate before deployment without access to labelled anomalies.

\paragraph{Deep Learning-based Anomaly Detection}

Deep models have improved state-of-the-art performance in anomaly detection for highly structured, high dimensional data especially. Autoencoders are particularly popular, with the reconstruction error acting as the anomaly score. Normal samples are assumed to be reconstructed with lower error than anomalies. They have been used with fully-connected \cite{Sakurada2014}, convolutional \cite{Zhao2017d} or recurrent \cite{Malhotra2015} layers for different data applications. Variational \cite{An2015}, denoising \cite{Feng2015} and adversarial \cite{Vu2019} autoencoders have also been used. The reconstruction errors from each encoder-decoder layer pair are fused together in \cite{Kim2019}. The autoencoder latent encodings are optimised directly in \cite{Goodge} to improve robustness against adversarial perturbations.

Others use deep models as feature extractors for a secondary anomaly-detecting module, such as \textsc{KNN} \cite{bergman2020deep}, \textsc{KDE} \cite{nicolau2016hybrid}, \textsc{DBSCAN} \cite{amarbayasgalan2018unsupervised} or autoregressive models \cite{Abati2019}. Zong et al. \shortcite{Bo2018} simultaneously train an autoencoder for feature extraction with a Gaussian mixture model in the latent space for anomaly detection. Ruff et al. \shortcite{Ruff2018} learn a normality-encoding hypersphere in the latent space and the anomaly score is the distance from the centre. Generative adversarial networks use the ability of the generator to generate an unseen sample to indicate its anomalousness \cite{Schlegl2017,zenati2019}.

There has been some interest in GNNs for anomaly detection in graph data, such as sensor networks \cite{deng2021graph,cai2020structural,zheng2019addgraph}. Our method also uses GNNs, though it is distinct from these works as it is designed for unstructured, feature-based data rather than graphs.

\section{Background}\label{section:Background}

\subsection{Local Outlier Methods}\label{section:LOM}

`Local outlier methods' refers to those methods which directly use the distance of a point to its $k$ nearest neighbours to determine its anomalousness. We now detail \textsc{KNN} and \textsc{LOF}.

\paragraph{KNN} The anomaly score of a point $\mathbf{x}_i$ is its distance to its $k^{th}$ nearest neighbour:

\begin{equation}\label{eq:KNN}
    \textsc{KNN}(\mathbf{x}_i) = \mathsf{dist}(\mathbf{x}_i, \mathbf{x}_i^{(k)})
\end{equation}
where $\mathbf{x}_i^{(k)}$ is the $k^{th}$ nearest neighbour of $\mathbf{x}_i$. Euclidean distance is most common, though any distance measure could be used depending on its suitability to the data type.
 
\paragraph{LOF} The Local Outlier Factor instead uses the `reachability distance', which is defined for $\mathbf{x}_i$ from $\mathbf{x}_j$ as:

\begin{equation}\label{eq:reachability_distance}
\mathsf{reach}_{k}(\mathbf{x}_i,\mathbf{x}_j) =\mathsf{max}\{ \mathsf{k\text{-}dist}(\mathbf{x}_j), \mathsf{dist}(\mathbf{x}_i,\mathbf{x}_j)\}
\end{equation}
where $\mathsf{k\text{-}dist}(\mathbf{x}_j)$ is equal to $\mathsf{dist}(\mathbf{x}_j,\mathbf{x}_{j}^{(k)})$. This is used to calculate the `local reachability density` of a point: 

\begin{equation}\label{eq:lrd_score}
\mathsf{lrd}_k (A) := \left(\frac{\underset{j \in \mathcal{N}_i}{\sum} \mathsf{reach}_{k}(\mathbf{x}_i,\mathbf{x}_j)}{|\mathcal{N}_i|}\right)^{-1}
\end{equation}
where $\mathcal{N}_i$ is the set of $k$ nearest neighbours of $\mathbf{x}_i$.  Finally, this density measure is compared with that of neighbouring points to determine the local outlier factor:

\begin{equation}\label{eq:lof_score}
\textsc{LOF}(\mathbf{x}_i) = \frac{\sum_{j\in \mathcal{N}_i} \mathsf{lrd}_k (\mathbf{x}_j)}{|\mathcal{N}_i| \cdot \mathsf{lrd}_k (\mathbf{x}_i)}.
\end{equation}

\subsection{Graph Neural Networks}\label{section:GNN}

Graph neural networks (GNN) operate on a graph $G(V,E)$, in which a node, $i \in V$, is connected to an adjacent node, $j \in V$, via an edge $(j,i) \in E$. Edges can be undirected, in which case information flows in both directions between adjacent nodes. Alternatively, if the edges are directed, then information only flows from the source node to target node, i.e. from $j$ to $i$ along the edge $(j,i)$. Nodes and edges can, but need not, have feature vectors, denoted by $\mathbf{x}_i$ and $e_{j,i}$ for node $i$ and edge $(j,i)$ respectively.

GNNs have become increasingly popular in a range of graph-related applications, such as social networks \cite{fan2019graph} and traffic networks \cite{cui2019traffic}. Of particular interest here is the node classification task, which involves learning successive latent representations of nodes through the network layers in order to predict the class label of each node. This relies on a message passing scheme, made up of \textit{message}, \textit{aggregation} and \textit{update} steps. The message function ($\phi$) determines the information to be sent to the node in question from each neighbour. The aggregation function ($\square$) summarises these incoming messages into one message, for example by average or max-pooling. Finally, the update function ($\gamma$) uses this aggregated message and the current representation of the node to compute its subsequent representation. In summary, the $k^{th}$ layer of a GNN calculates the hidden representation of a node via the following \cite{gilmer2017neural}:

\begin{align}\label{eq:GNN}
& \mathbf{h}_{\mathcal{N}_i}^{(k)} = \underset{j  \in \mathcal{N}_i}{\square} \phi^{(k)} ( \mathbf{h}_i^{(k-1)},\mathbf{h}_j^{(k-1)}, \mathbf{e}_{j,i} ), \nonumber \\
& \mathbf{h}_i^{(k)} = \gamma^{(k)} (\mathbf{h}_i^{(k-1)},\mathbf{h}_{\mathcal{N}_i}^{(k)}). 
\end{align}
where $h_i^{(0)} = \mathbf{x}_i$ and $\mathcal{N}_i$ is the set of adjacent nodes to $i$. $\mathbf{h}_{\mathcal{N}_i}^{(k)}$ is the aggregation of the messages from its neighbours.

\section{Problem Definition}

We now define the unsupervised anomaly detection problem of interest in this paper. We assume to have $m$ normal training samples $\mathbf{x}_1^{\text{(train)}},..., \mathbf{x}_m^{\text{(train)}} \in \mathbb{R}^d$ and $n$ testing samples, $\mathbf{x}_1^{\text{(test)}}, ..., \mathbf{x}_n^{\text{(test)}} \in \mathbb{R}^d$, each of which may be normal or anomalous. For a test sample $\mathbf{x}_i^{\text{(test)}}$, our algorithm should output an \textbf{anomaly score} $s(\mathbf{x}_i^{\text{(test)}})$ that is low (or high) if $\mathbf{x}_i^{\text{(test)}}$ is normal (or anomalous).

In local outlier methods, the fundamental question is:

\textit{How should the distances of a sample $\mathbf{x}_i^{\text{(test)}}$ to its nearest neighbours be used in computing its anomaly score?}

In the following section, we show that many local outlier methods can be seen as particular cases of the message passing framework used by GNNs. 

\section{Unifying Framework}\label{sec:framework}

Local outlier methods collect information from the nearest neighbouring points to compute a statistic to indicate the anomalousness of a given point. This process fits within the GNN message passing framework outlined in (\ref{eq:GNN}). For ease of understanding, we show this using the example of \textsc{KNN} in particular.

\subsection*{Example: \textsc{KNN}}

Recall that \textsc{KNN} computes the anomaly score based on the distance to the $k^{th}$ nearest neighbour of a point. 

In the context of message passing, each data sample corresponds to one node in a graph and node $i$ is connected to each of its $k$ nearest neighbours, $j \in \mathcal{N}_i$, via a directed edge $(j,i)$, with edge feature $\mathbf{e}_{j,i}$ equal to the distance between them ($k$-NN graph):

\begin{equation}
    e_{j,i} = \begin{cases}
    \mathsf{dist}(\mathbf{x}_i,\mathbf{x}_j) \ \text{if} \ j \in \mathcal{N}_i. \\
    0 \ \text{otherwise}.
    \end{cases}
\end{equation}
These edges are directed as $j \in \mathcal{N}_i \centernot\implies i \in \mathcal{N}_j$, so information flows along edge $(j,i)$ only from the source node $j$ to target node $i$. With this graph, we now show that \textsc{KNN} can be explained in terms of the \textit{message}, \textit{aggregation} and \textit{update} functions in (\ref{eq:GNN}).

\paragraph{Message} \textsc{KNN} collects the distances of a node to its nearest neighbours:

\begin{equation}
    \phi^{(1)} := \mathbf{e}_{j,i}.
\end{equation}

\paragraph{Aggregation} It then outputs the maximum of these distances (i.e. max-pooling):

\begin{equation}
    \mathbf{h}_{\mathcal{N}_i}^{(1)}:= \underset{j\in \mathcal{N}_i}{\mathsf{max}} \ \phi^{(1)}
\end{equation}

\paragraph{Update}
Finally, it outputs this aggregated message as the anomaly score: 

\begin{equation}
\gamma^{(1)} := \mathbf{h}_{\mathcal{N}_i}^{(1)}
\end{equation}

\begin{prop}
\textsc{KNN} is a special case of the message passing scheme formulated in (\ref{eq:GNN}).
\end{prop}

\begin{proof}

The \textsc{KNN} anomaly score can be calculated using the \textit{message}, \textit{aggregation} and \textit{update} functions formulated above. By substituting these functions into their appropriate counterparts in (\ref{eq:GNN}), we arrive at the following:

\begin{equation}
    \textsc{KNN}(i) = \underset{j\in \mathcal{N}_i}{\mathsf{max}}(\mathbf{e}_{j,i}),
\end{equation}
which is a special, one-layer case of the message passing framework in (\ref{eq:GNN}).
\end{proof}

A similar analysis can be applied to \textsc{LOF} and \textsc{DBSCAN}, which are instead two-layer cases with two rounds of message passing. For example, in \textsc{LOF}, the first layer calculates the local reachability density as in (\ref{eq:lrd_score}) and the second layer calculates the local outlier score as in (\ref{eq:lof_score}). Table \ref{tab:LOF} formalizes these connections, and an extended version with more local outlier methods can be found in the supplementary material.

\begin{table}
\centering
\begin{tabular}{cccc}
\toprule
Step & \textsc{KNN} & \textsc{LOF} & \textsc{DBSCAN} \\
\midrule
$\mathbf{e}_{j,i}$ & $\mathsf{dist}(\mathbf{x}_i,\mathbf{x}_j)$ & $\mathsf{reach}(\mathbf{x}_i,\mathbf{x}_j)$ & $\mathsf{dist}(\mathbf{x}_i,\mathbf{x}_j)$\\
$\phi^{(1)}$ & $\mathbf{e}_{j,i}$ & $\mathbf{e}_{j,i}$ & $H(\epsilon - \mathbf{e}_{j,i})$ \\
$\square^{(1)}$ & $\mathsf{max}$ & $\mathsf{sum}$ & $\mathsf{sum}$\\
$\gamma^{(1)}$ & $\mathbf{h}_{\mathcal{N}_i}^{(1)}$ & $1/\mathbf{h}_{\mathcal{N}_i}^{(1)}$ & $H(\mathbf{h}_{\mathcal{N}_i}^{(1)} -$ minPts) \\
$\phi^{(2)}$ & - & $\mathbf{h}_{j}^{(1)}/\mathbf{h}_{i}^{(1)}$ & $\mathbf{h}_j^{(1)}$\\
$\square^{(2)}$ & - & $\mathsf{mean}$ & $\mathsf{max}$\\
$\gamma^{(2)}$ & - & $\mathbf{h}_{\mathcal{N}_i}^{(2)}$ &  $1 - \mathbf{h}_{\mathcal{N}_i}^{(2)}$\\
\bottomrule
\end{tabular}
\caption{Local outlier methods as they relate to the message passing framework defined in (\ref{eq:GNN}). $H$ refers to the Heaviside function.}
\label{tab:LOF}
\end{table}

\section{Motivation: The Importance of Learnability}

Local outlier methods lack trainable parameters which enable them to optimise their performance for a given training set. In this section, we show that this hinders their overall accuracy. To do this, we compare the performance of \textsc{LOF} against our novel methodology \textsc{LUNAR}, on a toy training dataset of $1000$ points sampled from four Gaussian distributions. As perfectly pure training sets are rare in practice, we also generate $15$ points from a uniform distribution within the data bounds. These points are much rarer and sparser than the others, so they should not significantly influence the predicted normal regions.

\begin{figure}
    \centering
    \includegraphics[width=\linewidth]{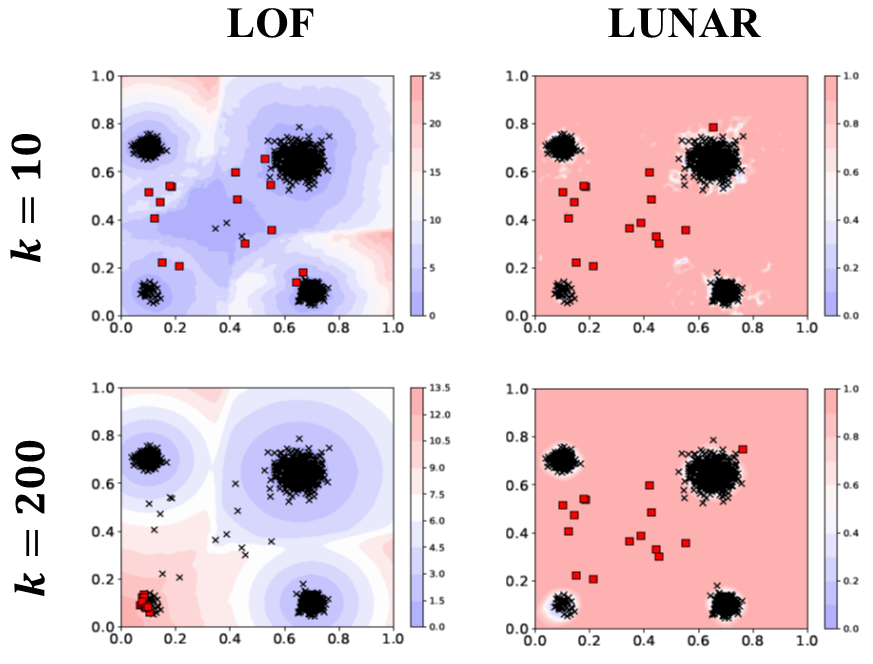}
    \caption{Contours of the scores assigned by \textsc{LOF} versus \textsc{LUNAR}. Red indicates a high (anomalous) score and blue indicates a low (normal) score. Points are marked by black crosses and those with the top $15$ highest assigned scores are marked by red squares.}
    \label{fig:contour}
\end{figure}

In Figure \ref{fig:contour}, low scores (blue) indicate a predicted normal region while high scores (red) indicate a predicted anomalous region. The points with one of the top $15$ anomaly scores are indicated with red squares. We test the methods with a small and large value for the hyperparameter dictating the number of nearest neighbours ($k$).

With low $k$, the \textsc{LOF} score is low around the four clusters, but also low far away from these clusters with little or no nearby training points. The central outlying region especially appears normal due to the strong influence of the relative sparsity of the very few points in the area. Conversely, with large $k$, the \textsc{LOF} score is erroneously high for the smaller cluster in the bottom-left corner. \textsc{LOF} fails to recognise the clusters existence as it contains fewer points than $k$, instead predicting all nearby points to be anomalous. These issues are challenging as local outlier methods lack the capacity to learn a more optimal scoring mechanism from the data directly. 

In comparison, the learnability of \textsc{LUNAR} enables it to perform better and more robustly across $k$: the regions assigned with normal or anomalous scores are a much closer fit to the training data, and the highest anomaly scores are given to the sparse, central points more accurately. We now describe its methodology in full.

\section{\textsc{LUNAR}: Methodology}\label{sec:lunar}

\paragraph{Overview}

Our methodology involves a one layer graph neural network as per the message passing framework described by (\ref{eq:GNN}). We represent a set of data as a graph, with a node corresponding to each data sample and directed edges connecting a target node to a set of source nodes, which are the nearest neighbours of the samples. For a given target node, the network utilises information from its its nearest neighbouring nodes to learn its anomaly score. It differs from other GNN implementations for several reasons:

\begin{itemize}
    \item We construct the $k$-NN graph of any feature-based, tabular dataset, rather than being restricted to graph datasets.
    \item We use a node's distances to its $k$ nearest neighbours as input, which is more generalizable than using its feature vector.
    \item We use a learnable message aggregation function, whereas most GNNs use a fixed aggregation approach.
\end{itemize}

\subsection{Model Design}

We now describe the methodology used in \textsc{LUNAR} in more detail, starting with how the graph is formulated.

\paragraph{Nearest Neighbourhood Graph}

For a data sample $\mathbf{x}_i$, we define a target node $i$ and edge $(j,i)$ connecting it to a source node $j$ for all $j$ where $\mathbf{x}_j$ is in the set of $k$ nearest neighbours to $\mathbf{x}_i$. The edge feature vector is equal to the Euclidean distance between the two points:

\begin{equation}\label{eq:knn_graph}
    \mathbf{e}_{j,i} = \begin{cases}
    \mathsf{dist}(\mathbf{x}_i,\mathbf{x}_j) \ \text{if} \ j \in \mathcal{N}_i. \\
    0 \ \text{otherwise}.
    \end{cases}
\end{equation}

As training samples are all assumed to be normal, we only search for nearest neighbours among training samples,  so that anomalies cannot influence the neighbourhood. With this, we define the \textit{message}, \textit{aggregation} and \textit{update} functions in (\ref{eq:GNN}) as follows:

\paragraph{Message}
The message passed from source node $j$ to target node $i$ along edge $(j,i)$ is equal to the edge feature $e_{j,i}$ (i.e. the distance between the points):

\begin{equation}
\phi^{(1)} := \mathbf{e}_{j,i}.
\end{equation}

\paragraph{Aggregation}
Rather than a fixed average or max-pooling, we use a learnable aggregation, which is suitable for our setting as we are dealing with node neighbourhoods of a fixed size ($k$). Our message aggregation involves concatenating them to give a $k$-dimensional vector, $\mathbf{e}^{(i)}$, where each entry represents the distance of $\mathbf{x}_i$ to its corresponding neighbour:

\begin{equation}\label{eq:distance_vector}
\mathbf{e}^{(i)} := [\mathbf{e}_{1,i},...,\mathbf{e}_{k,i}] \in \mathbb{R}^{k}.
\end{equation}

This vector is mapped to a single, scalar value representing the anomalousness of node $i$, through a neural network: 

\begin{equation}\label{eq:aggr_network}
\mathbf{h}_{\mathcal{N}_i}^{(1)} := \mathcal{F}(\mathbf{e}^{(i)},\Theta),
\end{equation}
where $\Theta$ are the weights of the neural network $\mathcal{F}$.

\paragraph{Update}
Finally, the update function outputs this learned, aggregated message:

\begin{equation}
    \gamma^{(1)} := \mathbf{h}_{\mathcal{N}_i}^{(1)}.
\end{equation}
We use a loss function which trains the GNN to output a score of $0$ for normal nodes and $1$ for anomalous nodes. As all training points are of the normal class, the network would attain perfect training accuracy by outputting zero scores regardless of the input. To avoid this trivial solution, we generate negative samples to act as artificial anomalies, training the model to output a score of $1$ for the negative sample nodes. With this, we aim to learn a decision boundary between normal samples and negative samples which generalises to the true anomalies in the test set. In the next section, we detail how negative samples are generated.

\subsection{Negative Sampling}\label{sec:negative_sampling}

\begin{table}
    \centering
    \begin{tabular}{ccccc}
    \toprule
    Dataset & \#Size & \#Dim & \#Anomalies\\
    \midrule
    \texttt{HRSS} & $90515$ & $20$ & $10187$ \\
    \texttt{MI-F} & $24955$ & $58$ & $2050$ \\
    \texttt{MI-V} & $22905$ & $58$ & $3942$ \\
    \texttt{OPTDIGITS} & $5216$ & $64$ & $150$ \\
    \texttt{PENDIGITS} & $6870$ & $16$ & $156$ \\
    \texttt{SATELLITE} & $6435$ & $36$ & $399$ \\
    \texttt{SHUTTLE} &  $49097$ & $9$ & $3511$ \\
    \texttt{THYROID} & $7200$ & $21$ & $534$\\
    \bottomrule
    \end{tabular}
    \caption{Statistics of the datasets used in experiments.}
    \label{tab:datasets}
\end{table}

\begin{table*}[ht!]
    \centering
    \begin{tabular}{rcccccccccc}
    \toprule
     Dataset & \textsc{IForest} & \textsc{OC-SVM} & \textsc{LOF} & \textsc{KNN} & \textsc{AE} & \textsc{VAE} & \textsc{DAGMM} & \textsc{SO}-GAAL & \textsc{DN2} & \textsc{LUNAR}\\
     \midrule
     \texttt{HRSS} & 59.61 & 61.03 & 60.13 & 62.09 & 61.16 & 63.30 & 55.93 & 45.90 & 60.20 & \ \ \ \  \textbf{92.17}** \\
     \texttt{MI-F} & 84.24 & 78.65 & 63.07 & 78.08 & 71.53 & 78.63 & 81.45 & 32.07 & 77.26 & \textbf{84.37} \\
     \texttt{MI-V} & 84.28 & 74.56 & 79.14 & 82.71 & 82.42 & 75.96 & 78.19 & 55.34 & 62.54 & \ \ \ \ \textbf{96.73}** \\ 
     \texttt{OPTDIGITS} & 79.34 & 59.84 & 99.53 & 96.57 & 97.46 & 86.71 & 75.56 & 74.35 & 34.98 & \textbf{99.76} \\
     \texttt{PENDIGITS} & 96.70 & 94.08 & 98.18 & 98.42 & 96.42 & 94.76 & 95.98 & 94.65 & 85.30 & \ \ \ \ \textbf{99.81}** \\
     \texttt{SATELLITE} & 80.10 & 64.64 & 84.25 & \textbf{86.07} & 81.48 & 66.09 & 78.22 & 84.16 & 75.37 & 85.35 \\
     \texttt{SHUTTLE} & 99.64 & 98.29 & 99.80 & 99.56 & 99.26 & 98.33 & 99.51 & 99.38 & 96.97 & \ \ \ \ \textbf{99.97}** \\
     \texttt{THYROID} & 76.30 & 52.81 & 68.67 & 63.01 & 64.34 & 51.54 & 70.91 & 60.13 & 58.09 & \ \ \ \ \textbf{85.44}** \\
    \bottomrule
    \end{tabular}
    \caption{AUC Score for each method on each dataset. Best scores are highlighted in bold. Average scores marked by  ** are greater than the next best performing method with significance level $p<0.01$, according to the $t$-test. More significance test results are found in the supplementary material.}
    \label{tab:main_results}
\end{table*}

Negative samples have been used to introduce supervision to unsupervised tasks, such as in contrastive learning \cite{chen2020simple}, as well as anomaly detection \cite{sipple2020interpretable}. They need to be sufficiently distinguishable from normal samples for the model to learn the decision boundary, but not so dissimilar that the task is too easy and the learnt boundary fails to discriminate normal samples from real anomalies. With this in mind, we combine two methods of generating negative samples, which are as follows:

\paragraph{Uniform}
The first method involves generating negative samples from a uniform distribution:
\begin{equation}
\mathbf{x}^{\text{(negative)}} \sim \mathcal{U}(-\varepsilon, 1 + \varepsilon) \in \mathbb{R}^d,
\end{equation}
where $\varepsilon$ is a small, positive constant. For simplicity, we use $\epsilon = 0.1$ in all experiments. The training data is normalized to the range $[0,1]$, so these samples cover the data bounds. However, normal data occupies a much smaller subspace within these bounds, so many of these negative samples would be far from normal data and ineffective for learning the decision boundary. We complements this by generating an additional set of more `difficult` negative samples.

\paragraph{Subspace Perturbation}

In the second method, we generate negative samples by adding Gaussian noise to normal samples in a subset of their feature dimensions:

\begin{align}
\mathbf{z} &\sim \mathcal{N}(\mathbf{0},I) \in \mathbb{R}^{d}, \nonumber \\
\mathbf{x}^{\text{(negative)}}_{i} &= \mathbf{x}^{\text{(train)}}_{i} + \mathbf{M} \circ \varepsilon \mathbf{z}.
\end{align}
where $\varepsilon$ is a small, positive constant and $\mathbf{M} \in \mathbb{R}^{d}$ is a vector of binary random variables. Each element in $\mathbf{M}$ has probability $p$ of being one (and $1-p$ of being zero), which determines the feature dimensions to be perturbed. We use $p=0.3$ in all experiments.

\paragraph{Computational Runtime}

In the supplementary material, we show the runtimes of \textsc{LUNAR} versus other methods in experiments. We see that \textsc{LUNAR} is faster than the other deep methods tested (e.g. $33.71$ seconds for \textsc{LUNAR} versus $55.92$ seconds for \textsc{DAGMM} on the \texttt{HRSS} dataset). \textsc{LUNAR} avoids directly training on high-dimensional feature data in its input, instead using distances between points, which explains the faster training time.

\paragraph{Limitations}

A limitation of \textsc{LUNAR}, as with all local outlier methods, is in finding the $k$ nearest neighbours. This is mostly an issue in very high-dimensional spaces, such as with image data, where distance measures become less meaningful \cite{Beyer1999}. Adapting \textsc{LUNAR} for higher dimensionality is left for future work at present.

\paragraph{Theoretical Properties}

An additional benefit of our unified approach is that we can use it to characterize theoretical properties of almost all local outlier methods in our framework (including \textsc{LUNAR}, \textsc{KNN}, \textsc{LOF}, and \textsc{DBSCAN}) in a unified way. One simple but important property of algorithms is their \emph{symmetries under transformations}, which are very relevant to understanding their \emph{inductive biases}, or the assumptions they use to generalize to unseen data. 

Let $s(\mathbf{x}; \{\mathbf{x}_i^{(\textrm{train})}\}_{i=1}^m)$ be the anomaly score of any local outlier method evaluated at $\mathbf{x}$ given training data $\{\mathbf{x}_i^{(\textrm{train})}\}_{i=1}^m$. 

\begin{prop}[Transformation Equivariance]
Given any distance-preserving transformation $f$, the score $s$ is \emph{transformation equivariant}; that is,
\begin{align}
    s(\mathbf{x}; \{\mathbf{x}_i^{(\textup{train})}\}_{i=1}^m) = s(f(\mathbf{x}); \{f(\mathbf{x}_i^{(\textup{train})})\}_{i=1}^m)
\end{align}
\end{prop}
For example, $s$ is equivariant to rotations, translations and reflections. 
\begin{proof}
As shown in Table \ref{tab:LOF}, all these methods compute distances $\mathsf{dist}(\mathbf{x}_i,\mathbf{x}_j)$ or $\mathsf{reach}(\mathbf{x}_i,\mathbf{x}_j)$ as input, and do not use the input features $\mathbf{x}_i$ in any other way. Applying $f$ to the training and test data does not change the (reachability) distances between them, thus also preserving the score $s$.
\end{proof}

\section{Experiments}

We now conduct experiments with real datasets to answer the following research questions:

\noindent \textbf{RQ1 (Accuracy):} Does \textsc{LUNAR} outperform existing baselines in detecting true anomalies? \\
\textbf{RQ2 (Robustness):} Is \textsc{LUNAR} more robust to changes in the neighbourhood size, $k$, than existing local outlier methods?\\
\textbf{RQ3 (Ablation Study):} How do variations in our methodology affected its performance?\\

\subsection{Datasets}

Each dataset used in our experiments is publicly available and consists of a normal ($0$) class and anomaly ($1$) class. Table \ref{tab:datasets} summarises them and their key statistics.

As we focus on the unsupervised case, in which the training set only consists of samples labelled as normal (all labelled anomalies are in the test set). We use Area-Under-Curve (AUC) to measure performance. The relative proportion of anomalies in the test set does not affect the scoring of any individual point, so we can choose to randomly subsample normal points to achieve a 50:50 normal:anomaly ratio in the test set. Of the remaining normal samples, they are split 85:15 into a training set and validation set. We randomly generate both `Uniform' and `Subspace Perturbation' negative samples for the training and validation sets separately to avoid leaking information. We use a 1:1 ratio of negative:normal samples in both sets for all experiments.

\subsection{Training Procedure}

\begin{table*}[ht]
    \centering
    \begin{tabular}{r|cccc|cccc|cccc}
    \toprule
    \multicolumn{1}{r|}{$k$} & \textsc{LOF} & \textsc{KNN} & \textsc{DN2} & \textsc{LUNAR} & \textsc{LOF} & \textsc{KNN} & \textsc{DN2} & \textsc{LUNAR} & \textsc{LOF} & \textsc{KNN} & \textsc{DN2} & \textsc{LUNAR} \\ 
    \midrule
    \multicolumn{1}{c}{} & \multicolumn{4}{c}{\texttt{HRSS}} & \multicolumn{4}{c}{\texttt{MI-F}} & \multicolumn{4}{c}{\texttt{MI-V}} \\
    \midrule
    2   & 82.08 & 86.25 & 85.28 & \textbf{93.88} & 90.43 & 77.84 & \textbf{91.13} & 81.50 & 94.31 & 94.58 & 86.76 & \textbf{96.06} \\
    10  & 67.98 & 65.53 & 62.40 & \textbf{92.67} & \textbf{86.41} & 73.46 & 85.58 & 82.39 & 92.60 & 88.53 & 77.92 & \textbf{96.09} \\
    50  & 61.66 & 62.71 & 60.46 & \textbf{92.21} & 67.17 & 71.69 & 78.66 & \textbf{83.58} & 78.61 & 83.29 & 64.96 & \textbf{96.38} \\
    100 & 60.13 & 62.09 & 60.20 & \textbf{92.17} & 63.07 & 78.08 & 77.26 & \textbf{84.37} & 79.14 & 82.71 & 62.54 & \textbf{96.73}  \\
    150 & 57.22 & 61.81 & 60.14 & \textbf{91.61} & 60.60 & 80.79 & 76.33 & \textbf{82.82} & 80.73 & 82.86 & 61.77 & \textbf{96.53} \\
    200 & 55.59 & 61.86 & 60.22 & \textbf{90.09} & 70.89 & 82.85 & 75.93 & \textbf{84.47} & 81.75 & 82.65 & 61.67 & \textbf{96.30} \\
    \midrule
    \textbf{Avg.} & 64.11 & 67.10 & 64.79 & \textbf{92.11} & 73.10 & 77.45 & 80.82 & \textbf{83.19} & 84.52 & 85.77 & 69.27 & \textbf{96.35} \\
    \midrule
    \multicolumn{1}{c}{} & \multicolumn{4}{c}{\texttt{OPTDIGITS}} & \multicolumn{4}{c}{\texttt{PENDIGITS}} & \multicolumn{4}{c}{\texttt{SATELLITE}} \\
    \midrule
    2 & 99.58 & \textbf{99.91} & 50.90 & \textbf{99.91} & 99.37 & \textbf{99.84} & 81.08 & \textbf{99.84} & 85.05 & 87.72 & 80.16 & \textbf{87.80}\\
    10 & \textbf{99.92} & 99.63 & 45.84 & 99.79 & 99.67 & 99.77 & 80.74 & \textbf{99.82} & 85.38 & 86.77 & 79.43 & \textbf{87.83}\\
    50 & 99.72 & 98.41 & 39.23 & \textbf{99.81} & 98.79 & 98.79 & 81.83 & \textbf{99.80} & 83.44 & 86.07 & 76.52 & \textbf{87.58}\\
    100 & 99.53 & 96.57 & 34.98 & \textbf{99.76} & 98.18 & 98.42 & 85.30 & \textbf{99.81} & 84.25 & \textbf{86.07} & 75.37 & 85.35 \\ 
    150 & 99.11 & 94.85 & 33.10 & \textbf{99.73} & 97.58 & 98.07 & 86.39 & \textbf{99.76} & 84.86 & \textbf{85.85} & 74.48 & 83.95 \\ 
    200 & 98.63 & 93.13 & 32.14 & \textbf{99.78} & 97.19 & 97.52 & 85.49 & \textbf{99.71} & 85.21 & \textbf{85.46} & 73.39 & 84.70  \\
    \midrule
    \textbf{Avg.} & 99.41 & 97.09 & 39.37 & \textbf{99.79} & 98.46 & 98.74 & 83.47 & \textbf{99.79} & 84.69 & \textbf{86.32} & 76.55 & 86.08 \\
    \midrule
    \multicolumn{1}{c}{} & & \multicolumn{1}{r|}{$k$} & \textsc{LOF} & \multicolumn{1}{c}{\textsc{KNN}} & \textsc{DN2} & \multicolumn{1}{c|}{\textsc{LUNAR}} & \textsc{LOF} & \multicolumn{1}{c}{\textsc{KNN}} & \multicolumn{1}{c}{\textsc{DN2}} & \multicolumn{1}{c}{\textsc{LUNAR}} & \\
    \cmidrule{3-11}
    \multicolumn{3}{c}{} & \multicolumn{4}{c}{\texttt{SHUTTLE}}  & \multicolumn{4}{c}{\texttt{THYROID}} & \\
    \cmidrule{3-11}
    \multicolumn{1}{c}{} & & \multicolumn{1}{r|}{2} & 99.64 & \multicolumn{1}{c}{\textbf{99.98}} & 98.94 & \multicolumn{1}{c|}{\textbf{99.98}} & \textbf{83.70} & \multicolumn{1}{c}{80.28} & \multicolumn{1}{c}{64.09} & \multicolumn{1}{c}{83.38} & \\
    \multicolumn{1}{c}{} & & \multicolumn{1}{r|}{10} & 99.91 & \multicolumn{1}{c}{99.93} & 98.22 & \multicolumn{1}{c|}{\textbf{99.95}} & 83.69 & \multicolumn{1}{c}{73.87} & \multicolumn{1}{c}{62.71} & \multicolumn{1}{c}{\textbf{84.24}} & \\
    \multicolumn{1}{c}{} & & \multicolumn{1}{r|}{50} & 99.74 & \multicolumn{1}{c}{99.68} & 97.19 & \multicolumn{1}{c|}{\textbf{99.97}} & 74.41 & \multicolumn{1}{c}{66.49} & \multicolumn{1}{c}{59.88} & \multicolumn{1}{c}{\textbf{86.01}} & \\
    \multicolumn{1}{c}{} & & \multicolumn{1}{r|}{100} & 99.80 & \multicolumn{1}{c}{99.56} & 96.97 & \multicolumn{1}{c|}{\textbf{99.97}} & 68.67 & \multicolumn{1}{c}{63.01} & \multicolumn{1}{c}{58.09} & \multicolumn{1}{c}{\textbf{85.44}} & \\ 
    \multicolumn{1}{c}{} & & \multicolumn{1}{r|}{150} & 99.80 & \multicolumn{1}{c}{99.43} & 96.68 & \multicolumn{1}{c|}{\textbf{99.95}} & 67.20 & \multicolumn{1}{c}{62.26} & \multicolumn{1}{c}{56.86} & \multicolumn{1}{c}{\textbf{86.08}} & \\ 
    \multicolumn{1}{c}{} & & \multicolumn{1}{r|}{200} & 99.69 & \multicolumn{1}{c}{99.32} & 96.45 & \multicolumn{1}{c|}{\textbf{99.97}} & 66.58 & \multicolumn{1}{c}{61.24} & \multicolumn{1}{c}{56.26} & \multicolumn{1}{c}{\textbf{86.67}} & \\
    \cmidrule{3-11}
    \multicolumn{1}{c}{} & & \multicolumn{1}{r|}{\textbf{Avg.}} & 99.76 & \multicolumn{1}{c}{99.65} & 97.41 & \multicolumn{1}{c|}{\textbf{99.96}} & 74.04 & \multicolumn{1}{c}{67.86} & \multicolumn{1}{c}{59.65} & \multicolumn{1}{c}{\textbf{85.31}} & \\
    \cmidrule{3-11}
    \end{tabular}
    \caption{AUC Score of \textsc{LOF}, \textsc{KNN}, \textsc{DN2} and \textsc{LUNAR} for different values of $k$ and the Avg. over all $k$. Best performance for each is highlighted in bold.}
    \label{tab:k_results}
\end{table*}

The neural network, $\mathcal{F}$ in (\ref{eq:aggr_network}), consists of four fully connected hidden layers all of size $256$. All layers used $\mathsf{tanh}$ activation except for the $\mathsf{sigmoid}$ function at the output layer. We used mean squared error as the loss function and Adam \cite{kingma2014adam} for optimization with a learning rate of $0.001$ and weight decay of $0.1$.

We trained the model for $200$ epochs and used the model parameters with the best validation score as the final model. It was implemented using PyTorch Geometric on Windows OS and a Nvidia GeForce RTX 2080 Ti GPU.

\subsection{Baselines}

We use the PyOD library \cite{zhao2019pyod} implementations of \textsc{IForest}, \textsc{OC-SVM}, \textsc{LOF}, \textsc{KNN}, and the GAN-based \textsc{SO-GAAL} \cite{liu2019generative}. We also implement a deep autoencoder (\textsc{AE}) and \textsc{VAE} built in Pytorch, and \textsc{DAGMM} as in \cite{Bo2018} with publicly available codes. Finally, \textsc{DN2} \cite{bergman2020deep}, which performs \textsc{KNN} with latent features learnt from a deep, pre-trained feature extractor. As we are interested in tabular data rather than image data, unlike the original paper, we use an autoencoder (the same model as in \textsc{AE}) for feature extraction.

\subsection{RQ1 (Accuracy):}

Table \ref{tab:main_results} shows the AUC score (multiplied by 100) of each method for each dataset. We use AUC as it does not rely on a user-defined score threshold to predict normal or anomalous labels. The scores shown are the average over five repeated trials with different random seeds. For the methods that use it, all results are with $k=100$ as the number of neighbours unless stated otherwise.

We see that \textsc{LUNAR} gives the best performance on all datasets except \texttt{SATELLITE}, for which \textsc{KNN} is slightly better. For the \texttt{HRSS}, \texttt{MI-V} and \texttt{THYROID} datasets in particular, our method performs substantially better than the baselines: between $10$ and $30$ percentage points better than the second best method. Our scores marked by ** are significantly better than the second best performing method for each dataset with significance value $p<0.01$ according to the $t$-test.
 
\subsection{RQ2 (Robustness to Neighbourhood Size):}

\textsc{LOF}, \textsc{KNN} and \textsc{DN2} also use the $k$ nearest neighbours of a point to determine its anomalousness. In Table \ref{tab:k_results}, we show the performance of these methods for various $k$. We see that these methods depend greatly on the value of $k$. For example, their score decreases by $26$, $24$ and $25$ percentage points respectively for \texttt{HRSS} as $k$ increases from $2$ to $200$. In stark contrast, \textsc{LUNAR} only drops in performance by $3$ points in the same range. \text{LUNAR} gives the best performance in the vast majority of datasets and $k$ settings. Our method not only performs better, but maintains stronger performance for different settings of $k$. This is because it is able to learn to use the information from all $k$ neighbours effectively, whereas the other methods lose information from most neighbours, as decided by a pre-set aggregation rule.

\subsection{RQ3 (Ablation Study):}

\begin{table}[h!]
    \centering
    \begin{tabular}{rccc}
    \toprule
    & \multicolumn{3}{c}{Negative sampling scheme} \\
    Dataset & \textsc{SP} & \textsc{U} & Mixed \\
    \midrule
    \texttt{HRSS} & \textbf{93.32} & 66.34 & 92.17  \\
    \texttt{MI-F} & 84.17 & 57.76 & \textbf{84.37} \\
    \texttt{MI-V} & 96.64 & 67.99 & \textbf{96.73} \\
    \texttt{OPTDIGITS} & 93.81 & \textbf{99.86} & 99.76 \\
    \texttt{PENDIGITS} & 99.78 & \textbf{99.82} & 99.81 \\
    \texttt{SATELLITE} & \textbf{85.37} & 85.12 & 85.35 \\
    \texttt{SHUTTLE} & 99.96 & 99.54 & \textbf{99.97} \\
    \texttt{THYROID} & \textbf{85.99} & 45.42 & 85.44 \\
    \bottomrule
    \end{tabular}
    \caption{AUC scores for different negative sample types.}
    \label{tab:neg_1_results}
\end{table}

Table \ref{tab:neg_1_results} shows the performance with Subspace Perturbation (\textsc{SP}) and Uniform (\textsc{U}) negative samples individually. \textsc{SP} negative samples give better performance than \textsc{U} samples except for the \texttt{OPTDIGITS} dataset, in which \text{SP} samples alone gives poor performance for small values of $k$. Overall, mixing both types gives the best performance in most cases.

Further ablation studies relating to the neural network size and depth, can be found in the supplementary material. Overall, we find that deeper and wider networks for message aggregation give the best performance.

\section{Conclusion}

We have studied local outlier methods, some of the most well-established and popular anomaly detection methods in practice, which use the distance of data samples to their nearest neighbours to detect anomalies. We provided a unifying framework which shows that many local outlier methods seen as particular cases of the message passing scheme used in graph neural networks.

We then proposed \textsc{LUNAR}, which is based on this shared framework but is also able to learn and adapt to different sets of data by using a graph neural network. We show that our method significantly outperforms the baselines, including other deep learning-based methods, on a wide variety of datasets. Our method also maintains its strong performance for different neighbourhood sizes much better than other local outlier methods, as it is unique in its ability to learn from all incoming information from the neighbours.

\section*{Acknowledgements}

This work was supported in part by NUS ODPRT Grant R252-000-A81-133.

\bibliography{ref}

\end{document}


\onecolumn

\maketitle

\section{Local Outlier Methods}

\begin{table}[h!]
\centering
\begin{tabular}{p{2.5cm}llllllp{2cm}}
\toprule
Method & $\mathbf{e}_{j,i}$ & $\phi^{(1)}$ & $\square^{(1)}$ & $\gamma^{(1)}$ & $\phi^{(2)}$ & $\square^{(2)}$ & $\gamma^{(2)}$ \\
\midrule
\textsc{\textbf{KNN}} (Anguilli,2002) & $\mathsf{dist}(i,j)$ & $\mathbf{e}_{j,i}$ & $\mathsf{max}$ & $\mathbf{h}_{N_i}$ & - & - & -\\

\textsc{\textbf{Aggr-KNN}} (Anguilli,2002)& $\mathsf{dist}(i,j)$ & $\mathbf{e}_{j,i}$ & $\mathsf{sum}$ & $\mathbf{h}_{N_i}$ & - & - & - \\

\textsc{\textbf{LOF}} (Breunig,2000) & $\mathsf{r}\text{-}\mathsf{dist}(i,j)$* & $\mathbf{e}_{j,i}$ & $\mathsf{sum}$ & $\mathbf{h}_{N_i}^{-1}$ & $\frac{\mathbf{h}_{j}}{\mathbf{h}_{i}}$ & $\mathsf{mean}$ & $\mathbf{h}_{N_i}$ \\

\textsc{\textbf{Simple-LOF}} (Schubert,2014) & $\mathsf{dist}(i,j)$ & $\mathbf{e}_{j,i}$ & $\mathsf{sum}$ & $\mathbf{h}_{N_i}^{-1}$ & $\frac{\mathbf{h}_{j}}{\mathbf{h}_{i}}$ & $\mathsf{mean}$ & $\mathbf{h}_{N_i}$ \\

\textsc{\textbf{LoOP}} (Kriegel, 2009a) $\mathsf{dist}(i,j)$ & $\mathbf{e}_{j,i}^2$ & $\mathsf{mean}$ & $\lambda \mathbf{h}_{N_i}$ & $\frac{\mathbf{h}_{j}}{\mathbf{h}_{i}} -1$ & $\mathsf{sum}$ & $\mathsf{max}\{0,\text{erf}(\frac{1}{\sqrt{2}}\mathbf{x}_{N_i})\}$ \\

\textsc{\textbf{INFLO}} (Jin,2006) & $\mathsf{dist}(i,j)$ & $\mathbf{e}_{j,i}^2$ & $k^{th}$-$\mathsf{max}$ & $\mathbf{h}_{N_i}^{-1}$ & $\frac{\mathbf{h}_{j}}{\mathbf{h}_{i}}$ & $\mathsf{mean}$ & $\mathbf{h}_{N_i}$ \\

\textsc{\textbf{DBSCAN}} (Ester,1996) & $\mathsf{dist}(i,j)$ & $H(\epsilon - \mathbf{e}_{j,i})$ & $\mathsf{sum}$ & $H(\mathbf{h}_{N_i} -$ minPts) & $\mathbf{h}_j$ & $\mathsf{max}$ & $1 - \mathbf{h}_{N_i}$ \\

\textsc{\textbf{ROS}} (Pei,2006) & $1$ & $\mathbf{h}_j - \mathbf{h}_i$ & $\mathsf{sum}$ & $1 - \mathsf{min}_{1\leq r \leq n}\mathbf{h}_{N_j}$ & - & - & - \\

\textsc{\textbf{SOD}} (Kriegel,2009b) & $1$ & $\mathbf{h}_j$ & $\mathsf{mean}$ & $\mathbf{h}_i -\mathbf{h}_{N_i}$ & $\mathbf{h}_j$ & $\mathsf{mean}$ &
$\frac{\mathbf{h}_{N_i} \cdot v^{R(i)}_c}{\Vert v^{R(i)}_c \Vert_1}$*** \\

\bottomrule
\end{tabular}
* $\mathsf{r}\text{-}\mathsf{dist}(i,j)$ is the reachability distance. \\
** erf is the Gauss error function \\ 
*** $v^{R(i)}_c = 1 \ \text{for} \ c = 1,..,d.$ \\
\end{table}

\onecolumn

\section{Experiments}

\subsection{Datasets}

\begin{table}[h!]
    \centering
    \begin{tabular}{rl}
    \toprule
    Dataset & Link \\
    \midrule
    \texttt{HRSS} & \text{https://www.kaggle.com/init-owl/high-storage-system-data-for-energy-optimization} \\
    \texttt{MI-F} & \text{https://www.kaggle.com/shasun/tool-wear-detection-in-cnc-mill} \\
    \texttt{MI-V} & \text{https://www.kaggle.com/shasun/tool-wear-detection-in-cnc-mill} \\
    \texttt{OPTDIGITS} & \text{http://odds.cs.stonybrook.edu (Rayana2016)} \\
    \texttt{PENDIGITS} & \text{http://odds.cs.stonybrook.edu (Rayana2016)}\\
    \texttt{SATELLITE} & \text{http://odds.cs.stonybrook.edu (Rayana2016)} \\
    \texttt{SHUTTLE} & \text{http://odds.cs.stonybrook.edu (Rayana2016)} \\
    \texttt{THYROID} & \text{http://odds.cs.stonybrook.edu (Rayana2016)} \\
    \bottomrule
    \end{tabular}
\end{table}

\subsection{Statistical Significance}

Table \ref{all_k} shows the $p$ values for which the performance of \textsc{LUNAR} is statistically significant over the local outlier methods (\textsc{LOF}, \textsc{KNN} and \textsc{DN2}) according to the one-sided Wilcoxon test, where all of the trials over all tested values of $k$ are considered. 

Table \ref{100k} shows the $p$ values for the which the performance of \textsc{LUNAR} is statistically significant over the second best performing baseline for each dataset respectively (with $k=100$ for local outlier methods specifically as in the main experiments). We perform the one-sided Wilcoxon test as well as T test. We see that overall, \textsc{LUNAR} significantly performs better than the other methods ($p$ < 0.01) for most datasets. The lowest $p$-value possible from Wilcoxon tests is limited by the number of trials, which explains so many values are at this lower limit. 

\begin{table*}[h!]
\centering
\begin{tabular}{rlll}
\toprule
Dataset & \textsc{LOF} & \textsc{KNN} & \textsc{DN2} \\
\midrule
\texttt{HRSS} & 0.00000012** & 0.00000012** & 0.00000012** \\
\texttt{MI-F} & 0.00037028** & 0.00000012** & 0.03093528* \\
\texttt{MI-V} & 0.00000012** & 0.00000012** & 0.00000012** \\ 
\texttt{OPTDIGITS} & 0.00020157** & 0.00000053** & 0.00000012** \\
\texttt{PENDIGITS} & 0.00000029** & 0.00000085** & 0.00000012** \\
\texttt{SATELLITE} & 0.00034885** & 0.64070398 & 0.00000094** \\
\texttt{SHUTTLE} & 0.00000012** & 0.00000014** & 0.00000012** \\
\texttt{THYROID} & 0.00000086** & 0.00000012** & 0.00000012** \\
\bottomrule
\end{tabular}
\caption{$p$-values of the significance of the improvement of \textsc{LUNAR} over the selected local outlier methods, calculated with the Wilcoxon one-sided test, over all values of $k$ tested in experiments. $p < 0.01$ is marked by ** and $p<0.05$ is marked by *.}
\label{all_k}
\end{table*}
    
\begin{table*}[h!]
\centering
\begin{tabular}{rll}
\toprule
 Dataset & Wilcoxon Test & T-test\\
 \midrule
 \texttt{HRSS} & 0.02155722* & 0.00000000** \\
 \texttt{MI-F} & 0.44636920 & 0.84219156 \\
 \texttt{MI-V} & 0.02155722* & 0.00000086** \\ 
 \texttt{OPTDIGITS} & 0.02155722* & 0.25849456 \\
 \texttt{PENDIGITS} & 0.02155722* & 0.00947708** \\
 \texttt{SATELLITE} & 0.97844278 & 0.33650830 \\
 \texttt{SHUTTLE} & 0.02155722* & 0.00002765** \\
 \texttt{THYROID} & 0.02155722* & 0.00004539**\\
\bottomrule
\end{tabular}
\caption{$p$-values of the significance of the improvement of \textsc{LUNAR} over the second best-performing baseline method, calculated with the Wilcoxon one-sided test and T-test, for $k=100$. $p < 0.01$ is marked by ** and $p<0.05$ is marked by *.}
\label{100k}
\end{table*}

\subsection{Standard Deviations}

Tables \ref{tab:main_stdev} and \ref{tab:k_stdev} show the standard deviations of the scores associated with each model and dataset, over the five trials performed for each.

\begin{table*}[ht!]
    \centering
    \begin{tabular}{rcccccccccc}
    \toprule
     Dataset & \textsc{IForest} & \textsc{OC-SVM} & \textsc{LOF} & \textsc{KNN} & \textsc{AE} & \textsc{VAE} & \textsc{DAGMM} & \textsc{SO-GAAL} & \textsc{DN2} & \textsc{LUNAR}\\
     \midrule
     \texttt{HRSS}      & 0.54 & 0.20 & 0.20 & 0.19 & 1.47 & 0.6 & 0.28 & 4.58 & 4.81 & 0.26 \\
     \texttt{MI-F}      & 1.22 & 0.37 & 2.58 & 0.50 & 9.06 & 0.6 & 0.66 & 1.85 & 3.86 & 0.44 \\
     \texttt{MI-V}      & 1.83 & 0.82 & 0.83 & 0.67 & 3.22 & 2.11 & 0.55 & 14.62 & 1.26 & 0.20 \\ 
     \texttt{OPTDIGITS} & 2.80 & 2.59 & 0.29 & 0.98 & 1.65 & 2.22 & 2.76 & 8.25 & 8.90 & 0.25 \\
     \texttt{PENDIGITS} & 1.13 & 1.49 & 0.41 & 0.78 & 2.29 & 2.31 & 1.12 & 2.17 & 8.44 & 0.25 \\
     \texttt{SATELLITE} & 2.22 & 1.53 & 0.88 & 0.75 & 2.25 & 3.36 & 1.08 & 1.30 & 9.33 & 1.18 \\
     \texttt{SHUTTLE}   & 0.05 & 0.01 & 0.03 & 0.01 & 0.48 & 0.02 & 0.20 & 0.03 & 1.31 & 0.02 \\
     \texttt{THYROID}   & 1.75 & 0.97 & 1.29 & 0.91 & 3.46 & 1.18 & 2.52 & 5.00 & 1.92 & 1.49 \\
    \bottomrule
    \end{tabular}
    \caption{Standard deviations of AUC scores for each method on each dataset.}
    \label{tab:main_stdev}
\end{table*}

\begin{table*}[ht!]
    \centering
    \begin{tabular}{r|cccc|cccc|cccc}
    \toprule
    \multicolumn{1}{r|}{$k$} & \textsc{LOF} & \textsc{KNN} & \textsc{DN2} & \textsc{LUNAR} & \textsc{LOF} & \textsc{KNN} & \textsc{DN2} & \textsc{LUNAR} & \textsc{LOF} & \textsc{KNN} & \textsc{DN2} & \textsc{LUNAR} \\
    \midrule
    \multicolumn{1}{c}{} & \multicolumn{4}{c}{\texttt{HRSS}} & \multicolumn{4}{c}{\texttt{MI-F}} & \multicolumn{4}{c}{\texttt{MI-V}} \\
    \midrule
    2   & 0.15 & 0.26 & 1.11 & 0.27 & 0.21 & 0.28 & 2.87 & 0.31 & 0.44 & 0.34 & 3.73 & 0.21 \\
    10  & 0.54 & 0.14 & 4.91 & 0.29 & 0.94 & 0.63 & 3.71 & 0.20 & 0.78 & 0.81 & 2.53 & 0.24 \\
    30  & 0.35 & 0.18 & 4.93 & 0.21 & 3.46 & 1.17 & 4.25 & 0.67 & 1.53 & 0.85 & 0.94 & 0.21 \\ 
    50  & 0.24 & 0.16 & 4.89 & 0.23 & 1.18 & 0.84 & 4.69 & 0.60 & 1.53 & 0.73 & 1.31 & 0.21 \\
    100 & 0.20 & 0.19 & 4.81 & 0.26 & 2.58 & 0.50 & 3.86 & 0.44 & 0.83 & 0.67 & 1.26 & 0.20  \\
    150 & 0.19 & 0.23 & 4.85 & 0.36 & 1.38 & 0.66 & 3.59 & 0.35 & 1.25 & 0.49 & 1.31 & 0.18 \\
    200 & 0.18 & 0.25 & 4.93 & 0.95 & 0.79 & 0.49 & 3.54 & 0.66 & 0.47 & 0.63 & 1.35 & 0.20 \\
    \midrule
    \multicolumn{1}{c}{} & \multicolumn{4}{c}{\texttt{OPTDIGITS}} & \multicolumn{4}{c}{\texttt{PENDIGITS}} & \multicolumn{4}{c}{\texttt{SATELLITE}} \\
    \midrule
    2 & 0.23 & 0.17 & 13.94 & 0.18 & 0.42 & 0.26 & 12.83 & 0.24 & 1.19 & 0.85 & 3.88 & 0.77\\
    10 & 0.12 & 0.21 & 12.65 & 0.18 & 0.21 & 0.25 & 13.62 & 0.25 & 1.28 & 0.87 & 6.78 & 4.08\\
    30 & 0.21 & 0.26 & 11.52 & 0.24 & 0.30 & 0.58 & 13.74 & 0.24 & 1.01 & 0.81 & 8.84 & 1.32\\ 
    50 & 0.23 & 0.44 & 10.35 & 0.18 & 0.38 & 0.67 & 12.97 & 0.25 & 0.94 & 0.82 & 9.22 & 1.50\\
    100 & 0.29 & 0.98 & 8.90 & 0.25 & 0.41 & 0.78 & 8.44 & 0.25 & 0.88 & 0.75 & 9.33 & 1.18 \\ 
    150 & 0.39 & 1.43 & 9.27 & 0.20 & 0.57 & 0.90 & 4.70 & 0.25 & 0.81 & 0.72 & 9.00 & 1.45 \\ 
    200 & 0.43 & 1.66 & 9.22 & 0.21 & 0.72 & 1.00 & 4.47 & 0.26 & 0.84 & 0.75 & 8.74 & 1.56  \\
    \midrule
    \multicolumn{1}{c}{} & & \multicolumn{1}{r|}{k} & \textsc{LOF} & \multicolumn{1}{c}{\textsc{KNN}} & \textsc{DN2} & \multicolumn{1}{c|}{\textsc{LUNAR}} & \textsc{LOF} & \multicolumn{1}{c}{\textsc{KNN}} & \multicolumn{1}{c}{\textsc{DN2}} & \multicolumn{1}{c}{\textsc{LUNAR}} & \\
    \cmidrule{3-11}
    \multicolumn{3}{c}{} & \multicolumn{4}{c}{\texttt{SHUTTLE}}  & \multicolumn{4}{c}{\texttt{THYROID}} & \\
    \cmidrule{3-11}
    \multicolumn{1}{c}{} & & \multicolumn{1}{r|}{2} & 0.02 & \multicolumn{1}{c}{0.01} & 1.57 & \multicolumn{1}{c|}{0.01} & 2.18 & \multicolumn{1}{c}{1.33} & \multicolumn{1}{c}{2.81} & \multicolumn{1}{c}{1.43} & \\
    \multicolumn{1}{c}{} & & \multicolumn{1}{r|}{10} & 0.01 & \multicolumn{1}{c}{0.01} & 1.69 & \multicolumn{1}{c|}{0.01} & 1.69 & \multicolumn{1}{c}{1.41} & \multicolumn{1}{c}{1.63} & \multicolumn{1}{c}{2.34} & \\
    \multicolumn{1}{c}{} & & \multicolumn{1}{r|}{30} & 0.04 & \multicolumn{1}{c}{0.01} & 1.57 & \multicolumn{1}{c|}{0.02} & 1.51 & \multicolumn{1}{c}{1.08} & \multicolumn{1}{c}{1.14} & \multicolumn{1}{c}{3.08} & \\ 
    \multicolumn{1}{c}{} & & \multicolumn{1}{r|}{50} & 0.07 & \multicolumn{1}{c}{0.01} & 1.35 & \multicolumn{1}{c|}{0.01} & 1.53 & \multicolumn{1}{c}{1.04} & \multicolumn{1}{c}{1.38} & \multicolumn{1}{c}{1.31} & \\
    \multicolumn{1}{c}{} & & \multicolumn{1}{r|}{100} & 0.03 & \multicolumn{1}{c}{0.01} & 1.31 & \multicolumn{1}{c|}{0.02} & 1.29 & \multicolumn{1}{c}{0.91} & \multicolumn{1}{c}{1.92} & \multicolumn{1}{c}{1.49} & \\ 
    \multicolumn{1}{c}{} & & \multicolumn{1}{r|}{150} & 0.03 & \multicolumn{1}{c}{0.01} & 1.40 & \multicolumn{1}{c|}{0.03} & 0.89 & \multicolumn{1}{c}{1.03} & \multicolumn{1}{c}{1.98} & \multicolumn{1}{c}{0.87} & \\ 
    \multicolumn{1}{c}{} & & \multicolumn{1}{r|}{200} & 0.03 & \multicolumn{1}{c}{0.01} & 1.55 & \multicolumn{1}{c|}{0.02} & 0.75 & \multicolumn{1}{c}{1.02} & \multicolumn{1}{c}{2.09} & \multicolumn{1}{c}{1.13} & \\
    \cmidrule{3-11}
    \end{tabular}
    \caption{Standard deviations of AUC scores of \textsc{LOF}, \textsc{KNN}, \textsc{DN2} and \textsc{LUNAR} for different values of $k$}
    \label{tab:k_stdev}
\end{table*}

\subsection{Runtimes}

Table \ref{tab:main_runtime} shows the average runtime (in seconds) over the five trials for each method and dataset, with $k=100$ for the local outlier methods. Our method is faster than the other deep methods tested in all cases.

\begin{table*}[ht!]
    \centering
    \begin{tabular}{rcccccccccc}
    \toprule
     Dataset & \textsc{IForest} & \textsc{OC-SVM} & \textsc{LOF} & \textsc{KNN} & \textsc{AE} & \textsc{VAE} & \textsc{DAGMM} & \textsc{SO-GAAL} & \textsc{LUNAR}\\
     \midrule
     \texttt{HRSS} & 2.95 & 137.86 & 13.73 & 14.64 & 579.67 & 245.34 & 55.92 & 34.54 & 33.71 \\
     \texttt{MI-F} & 1.37 & 26.23 & 18.53 & 18.91 & 170.70 & 149.18 & 18.67 & 68.26 & 12.21 \\
     \texttt{MI-V} & 1.83 & 14.41 & 11.34 & 11.74 & 127.02 & 93.01 & 14.05 & 57.56 & 10.35 \\ 
     \texttt{OPTDIGITS} & 0.45 & 1.26 & 1.69 & 1.71 & 40.27 & 58.17 & 5.31 & 35.13 & 4.00 \\
     \texttt{PENDIGITS} & 0.38 & 0.86 & 0.59 & 0.59 & 68.44 & 76.55 & 6.19 & 38.20 & 4.00 \\
     \texttt{SATELLITE} & 0.33 & 0.54 & 0.42 & 0.46 & 41.27 & 12.18 & 4.41 & 5.30 & 3.88 \\
     \texttt{SHUTTLE} & 1.71 & 38.53 & 2.83 & 3.15 & 374.56 & 123.62 & 32.44 & 18.09 & 17.76 \\
     \texttt{THYROID} & 0.35 & 35.18 & 0.54 & 0.58 & 61.44 & 28.16 & 5.90 & 35.18 & 5.35 \\
    \bottomrule
    \end{tabular}
    \caption{Runtime of each method on each dataset. \textsc{DN2} is omitted as its runtime is virtually equivalent to \textsc{AE} as it uses the same model for feature extraction.}
    \label{tab:main_runtime}
\end{table*}

\vspace{3cm}

\subsection{Ablation Study: Negative Sampling}

Table \ref{tab:negative_sampling} shows the performance of \textsc{LUNAR} when different formulations of negative samples are used. \textsc{SP} refers to only 'Subspace Perturbation' samples while \textsc{U} refers to only 'Uniform' samples. 'Mixed' refers to combining both types.

\begin{table*}[h!]
    \centering
    \begin{tabular}{r|ccc|ccc|ccc|ccc}
    \toprule
    $k$ & $\textsc{SP}$ & $\textsc{U}$ & Mixed & $\textsc{SP}$ & $\textsc{U}$ & Mixed & $\textsc{SP}$ & $\textsc{U}$ & Mixed & $\textsc{SP}$ & $\textsc{U}$ & Mixed\\
    \midrule
    \multicolumn{1}{c}{} & \multicolumn{3}{c}{\texttt{HRSS}} & \multicolumn{3}{c}{\texttt{MI-F}} & \multicolumn{3}{c}{\texttt{MI-V}} & \multicolumn{3}{c}{\texttt{OPTDIGITS}} \\
    \midrule
    2   &   93.08 & 91.06 & 93.88 &     81.37 & 79.91 & 81.50 &     96.05 & 95.54 & 96.06 &    00.31 & 99.94 & 99.91 \\
    10  &   93.63 & 75.97 & 92.67 &     82.69 & 79.22 & 82.39 &     96.18 & 94.97 & 96.09 &    03.05 & 99.94 & 99.79 \\
    30  &   93.81 & 73.48 & 92.47 &     82.54 & 73.49 & 82.84 &     96.18 & 77.95 & 96.17 &    79.38 & 99.91 & 99.78 \\
    50  &   93.90 & 76.66 & 92.21 &     83.71 & 69.58 & 83.58 &     96.39 & 70.37 & 96.38 &    91.47 & 99.89 & 99.81 \\
    100 &   93.32 & 66.34 & 92.17 &     84.17 & 57.76 & 84.37 &     96.64 & 67.99 & 96.73 &    93.81 & 99.86 & 99.76 \\
    150 &   93.27 & 52.46 & 91.61 &     83.05 & 52.78 & 82.82 &     96.40 & 66.35 & 96.53 &    96.05 & 99.88 & 99.73 \\
    200 &   91.42 & 47.62 & 90.09 &     84.57 & 49.11 & 84.47 &     96.29 & 65.75 & 96.30 &    96.99 & 99.86 & 99.78 \\
    \midrule
    \multicolumn{1}{c}{} & \multicolumn{3}{c}{\texttt{PENDIGITS}} & \multicolumn{3}{c}{\texttt{SATELLITE}} & \multicolumn{3}{c}{\texttt{SHUTTLE}} & \multicolumn{3}{c}{\texttt{THYROID}}\\
    \midrule
    2   &   99.62 & 99.84 & 99.84 &    87.79 & 87.83 & 87.83 &     99.95 & 99.98 & 99.98 &    83.34 & 81.72 & 83.38 \\
    10  &   99.78 & 99.83 & 99.82 &    86.91 & 87.53 & 84.98 &     99.94 & 99.96 & 99.95 &    85.24 & 79.80 & 84.24 \\
    30  &   99.66 & 99.83 & 99.80 &    87.62 & 87.56 & 87.58 &     99.95 & 99.91 & 99.94 &    85.76 & 56.05 & 84.59 \\
    50  &   99.74 & 99.82 & 99.80 &    86.98 & 87.17 & 87.08 &     99.95 & 99.89 & 99.97 &    85.91 & 45.95 & 86.01 \\
    100 &   99.78 & 99.82 & 99.81 &    85.37 & 85.12 & 85.35 &     99.96 & 99.54 & 99.97 &    85.99 & 45.42 & 85.44 \\
    150 &   99.77 & 99.82 & 99.76 &    84.07 & 85.63 & 83.95 &     99.95 & 95.67 & 99.95 &    86.34 & 45.48 & 86.08 \\
    200 &   99.74 & 99.81 & 99.71 &    84.44 & 86.10 & 84.70 &     99.97 & 93.82 & 99.97 &    86.64 & 46.26 & 86.67 \\
    \bottomrule
    \end{tabular}
    \caption{Performance of \textsc{LUNAR} for each setting of $k$ and for different formulation strategies for negative sampling}
    \label{tab:negative_sampling}
\end{table*}

\subsection{Ablation Study: Hidden Layer Size}

Table \ref{tab:layer_sizes} shows the performance when we vary the size of the hidden layers used in the aggregation network in $\{64,128\}.$ A value of $256$ is used in the experiments shown in the main paper. The general pattern is that wider layers give better performance across the range of $k$ values for most datasets.

\begin{table}[h!]
\centering
\begin{tabular}{r|ccccccc}
\toprule
\multicolumn{8}{c}{Layer Size = 64} \\
\midrule
k                      & 2 & 10 & 30 & 50 & 100 & 150 & 200 \\
\midrule
\texttt{HRSS}          & 92.98 & 90.96 & 86.94 & 84.83 & 85.27 & 85.02 & 83.07 \\
\texttt{MI-F}          & 80.68 & 80.40 & 80.55 & 81.14 & 82.49 & 82.28 & 83.53 \\
\texttt{MI-V}          & 95.80 & 95.66 & 95.87 & 96.08 & 96.62 & 96.50 & 96.42 \\
\texttt{OPTDIGITS}     & 99.88 & 99.86 & 99.77 & 99.67 & 99.74 & 99.75 & 99.75 \\
\texttt{PENDIGITS}     & 99.84 & 99.83 & 99.81 & 99.80 & 99.81 & 99.55 & 99.53 \\
\texttt{SATELLITE}     & 87.81 & 87.42 & 87.72 & 87.04 & 85.04 & 84.05 & 83.92 \\
\texttt{SHUTTLE}       & 99.98 & 99.96 & 99.94 & 99.87 & 99.86 & 99.91 & 99.93 \\
\texttt{THYROID}       & 82.17 & 84.05 & 84.37 & 85.88 & 84.23 & 85.50 & 85.83 \\
\midrule
\multicolumn{8}{c}{Layer Size = 128} \\
\midrule
k             & 2 & 10 & 30 & 50 & 100 & 150 & 200 \\
\midrule
\texttt{HRSS} & 93.75 & 92.05 & 91.21 & 90.78 & 90.64 & 89.43 & 88.26 \\
\texttt{MI-F} & 81.33 & 81.82 & 82.27 & 82.59 & 83.40 & 82.69 & 84.23 \\
\texttt{MI-V} & 96.00 & 96.01 & 96.07 & 96.28 & 96.69 & 96.58 & 96.39 \\
\texttt{OPTDIGITS} & 99.88 & 99.76 & 99.72 & 99.73 & 99.68 & 99.72 & 99.73 \\
\texttt{PENDIGITS} & 99.84 & 99.83 & 99.80 & 99.81 & 99.81 & 99.69 & 99.62 \\
\texttt{SATELLITE} & 87.84 & 87.38 & 87.52 & 87.00 & 85.30 & 84.15 & 84.18 \\
\texttt{SHUTTLE} & 99.98 & 99.96 & 99.93 & 99.91 & 99.95 & 99.95 & 99.96 \\
\texttt{THYROID} & 83.00 & 84.26 & 85.34 & 85.30 & 85.05 & 85.77 & 86.52 \\
\bottomrule
\end{tabular}
\caption{Average AUC scores over five trials for \textsc{LUNAR} with different layer sizes/widths in the neural network.}
\label{tab:layer_sizes}
\end{table}

\subsection{Ablation Study: Network Depth}

In Table \ref{tab:network_depth}, we vary the depth of the network between $2$ and $3$ hidden layers. $4$ layers are used in the experiments shown in the main paper. The general pattern is that a deeper network gives better performance across the range of $k$ values for most datasets, with the difference between more significant for larger $k$.

\begin{table}[h!]
\centering
\begin{tabular}{r|ccccccc}
\toprule
\multicolumn{8}{c}{Depth = 2} \\
\midrule
k                      & 2 & 10 & 30 & 50 & 100 & 150 & 200 \\
\midrule
\texttt{HRSS} & 92.37 & 82.64 & 76.41 & 76.29 & 72.26 & 69.95 & 67.89 \\
\texttt{MI-F} & 81.30 & 78.17 & 77.87 & 78.54 & 80.32 & 81.15 & 82.69 \\
\texttt{MI-V} & 95.87 & 94.77 & 93.85 & 94.52 & 94.86 & 95.18 & 95.19 \\
\texttt{OPTDIGITS} & 99.89 & 99.90 & 99.68 & 99.78 & 99.55 & 99.58 & 99.58 \\
\texttt{PENDIGITS} & 99.84 & 99.83 & 99.81 & 99.80 & 99.80 & 99.76 & 99.74 \\
\texttt{SATELLITE} & 87.59 & 87.60 & 87.54 & 87.23 & 85.35 & 83.95 & 82.88 \\
\texttt{SHUTTLE} & 99.94 & 99.97 & 99.92 & 99.88 & 99.77 & 99.74 & 99.77 \\
\texttt{THYROID} & 82.39 & 80.61 & 79.91 & 80.19 & 81.09 & 80.16 & 80.59 \\
\midrule
\multicolumn{8}{c}{Depth = 3} \\
\midrule
k             & 2 & 10 & 30 & 50 & 100 & 150 & 200 \\
\midrule
\texttt{HRSS} & 93.39 & 91.58 & 88.67 & 86.57 & 86.87 & 85.55 & 84.80 \\
\texttt{MI-F} & 80.91 & 80.73 & 80.62 & 81.63 & 82.69 & 82.50 & 84.33 \\
\texttt{MI-V} & 95.88 & 95.81 & 95.85 & 96.03 & 96.55 & 96.39 & 96.21 \\
\texttt{OPTDIGITS} & 99.92 & 99.76 & 99.76 & 99.78 & 99.77 & 99.76 & 99.71 \\
\texttt{PENDIGITS} & 99.84 & 99.82 & 99.79 & 99.81 & 99.82 & 99.70 & 99.70 \\
\texttt{SATELLITE} & 87.85 & 86.47 & 87.55 & 87.03 & 85.35 & 84.19 & 83.81 \\
\texttt{SHUTTLE} & 99.98 & 99.96 & 99.91 & 99.87 & 99.88 & 99.92 & 99.93 \\
\texttt{THYROID} & 82.31 &  84.28 & 85.47 & 86.05 & 84.86 & 85.42 & 85.59 \\
\bottomrule
\end{tabular}
\caption{Average AUC scores over five trials for \textsc{LUNAR} with different layer counts/depth in the neural network.}
\label{tab:network_depth}
\end{table}

\vspace{3cm}

Angiulli, F.; and Pizzuti, C. 2002. Fast outlier detection in high dimensional spaces. In European conference on principles of data mining and knowledge discovery, 15–27. Springer.

Breunig, M. M.; Kriegel, H.-P.; Ng, R. T.; and Sander, J. 2000. LOF: identifying density-based local outliers. In Proceedings of the 2000 ACM SIGMOD international conference on Management of data, 93–104.

Ester, M.; Kriegel, H.-P.; Sander, J.; Xu, X.; et al. 1996. A density-based algorithm for discovering clusters in large spatial databases with noise. In kdd, volume 96, 226–231. Jin, W.; Tung, A. K. H.; Han, J.; and Wang, W. 2006. Ranking outliers using symmetric neighborhood relationship. In PAKDD, 577–593. Springer.

Kriegel, H.-P.; Kroger, P.; Schubert, E.; and Zimek, A. 2009a. LoOP: local outlier probabilities. In Proceedings of the 18th ACM conference on Information and knowledge management, 1649–1652.

Kriegel, H.-P.; Kroger, P.; Schubert, E.; and Zimek, A. 2009b. Outlier detection in axis-parallel subspaces of high dimensional data. In Pacific-asia conference on knowledge discovery and data mining, 831–838. Springer.

Pei, Y.; Zaiane, O. R.; and Gao, Y. 2006. An efficient reference-based approach to outlier detection in large datasets. In Sixth International Conference on Data Mining (ICDM’06), 478–487. IEEE.

Schubert, E.; Zimek, A.; and Kriegel, H.-P. 2014. Local outlier detection reconsidered: a generalized view on locality with applications to spatial, video, and network outlier detection. Data mining and knowledge discovery, 28(1): 190–237

Shebuti Rayana (2016).  ODDS Library [http://odds.cs.stonybrook.edu]. Stony Brook, NY: Stony Brook University, Department of Computer Science.